\documentclass[english,a4paper,12pt]{article}

\usepackage{amsxtra, amsfonts, amssymb, amstext, amsmath, mathtools}%, 
\usepackage{amsthm}
\usepackage{booktabs}
\usepackage{nicefrac}
\usepackage{xspace}
\usepackage{url}
\usepackage{graphics,color}
\usepackage[algo2e,ruled,vlined,linesnumbered]{algorithm2e}
    \SetKwInOut{Input}{Input}
    \SetKwInOut{Output}{Output}
    \ResetInOut{output}
    \SetKw{Break}{break}
    \SetKw{Breakall}{break all}
\usepackage{wrapfig}
\usepackage{lmodern}

\newtheorem{theorem}{Theorem}
\newtheorem{lemma}[theorem]{Lemma}

\newcommand{\NSGAtwo}{NSGA\nobreakdash-II\xspace}
\newcommand{\NSGA}{NSGA\nobreakdash-III\xspace}

\newcommand{\onemax}{\textsc{OneMax}\xspace}
\newcommand{\zeromax}{\textsc{ZeroMax}\xspace}

\newcommand{\oneminmax}{\textsc{OneMinMax}\xspace}

\newcommand{\threeOMM}{\textsc{3-OMM}\xspace}

\newcommand{\lotz}{\textsc{LOTZ}\xspace}

\newcommand{\mlotz}{$m$\textsc{LOTZ}\xspace}
\newcommand{\mcocz}{$m$\textsc{COCZ}\xspace}

\newcommand{\R}{\ensuremath{\mathbb{R}}}

\newcommand{\N}{\ensuremath{\mathbb{N}}} % ohne Null!!!

\newcommand{\zmin}{z^{\min}}
\newcommand{\zmax}{z^{\max}}

\let\originalleft\left
\let\originalright\right
\renewcommand{\left}{\mathopen{}\mathclose\bgroup\originalleft}
\renewcommand{\right}{\aftergroup\egroup\originalright}

\newcommand{\colvec}[2][.6]{%
  \scalebox{#1}{%
    \renewcommand{\arraystretch}{.6}%
    $\begin{pmatrix}#2\end{pmatrix}$%
  }
}

\let\oldsqrt\sqrt
\def\hksqrt{\mathpalette\DHLhksqrt}
\def\DHLhksqrt#1#2{\setbox0=\hbox{$#1\oldsqrt{#2\,}$}\dimen0=\ht0
   \advance\dimen0-0.2\ht0
   \setbox2=\hbox{\vrule height\ht0 depth -\dimen0}%
   {\box0\lower0.4pt\box2}}
\renewcommand\sqrt\hksqrt

\DeclareMathOperator{\refp}{rp}
\DeclareMathOperator{\nad}{nad}
\DeclareMathOperator{\mTrue}{True}
\DeclareMathOperator{\mFalse}{False}
\DeclareMathOperator{\valid}{valid}

\usepackage{hyperref}

\usepackage{amsxtra, amsthm, amsfonts, amssymb, amstext, amsmath, mathtools}%, 
\usepackage{booktabs}
\usepackage{cite}
\usepackage{nicefrac}
\usepackage{xspace}
\usepackage{url}
\usepackage{graphics,color}
\usepackage[algo2e,ruled,vlined,linesnumbered]{algorithm2e}
    \SetKwInOut{Input}{Input}
    \SetKwInOut{Output}{Output}
    \SetKw{Break}{break}
    \SetKw{Breakall}{break all}
\usepackage{wrapfig}
\usepackage{lmodern}
\usepackage{pgf}

\usepackage{placeins}

\begin{document}

%\begin{large}
\title{A Mathematical Runtime Analysis of the Non-dominated Sorting Genetic Algorithm III (\NSGA)}

%\author{Benjamin Doerr\\ Laboratoire d'Informatique (LIX)\\ CNRS\\ \'Ecole Polytechnique\\ Institut Polytechnique de Paris\\ Palaiseau\\ France%\\ email: {\tt doerr@mpi-inf.mpg.de}}
\author{Simon Wietheger\setcounter{footnote}{6}\thanks{Hasso Plattner Institute, Potsdam, Germany. Work done while visiting \'Ecole Poytechnique, France} \and Benjamin Doerr\thanks{Laboratoire d'Informatique (LIX), CNRS, \'Ecole Polytechnique, Institut Polytechnique de Paris, Palaiseau, France}}

%\author{Benjamin Doerr\thanks{Laboratoire d'Informatique (LIX), CNRS, \'Ecole Polytechnique, Institut Polytechnique de Paris, Palaiseau, France}}

%
%\title{Multiplicative Up-Drift\thanks{Extended and improved version of a paper that appeared in the proceedings of GECCO 2019~\cite{DoerrK19}. }
%}
%%\titlerunning{Multiplicative Up-Drift}        % if too long for running head

\maketitle

\sloppy{

\begin{abstract}
The Non-dominated Sorting Genetic Algorithm II (NSGA-II) is the most prominent multi-objective evolutionary algorithm for real-world applications.
While it performs evidently well on bi-objective optimization problems, empirical studies suggest that it is less effective when applied to problems with more than two objectives. A recent mathematical runtime analysis confirmed this observation by proving the NGSA-II for an exponential number of iterations misses a constant factor of the Pareto front of the simple 3-objective \oneminmax problem.

In this work, we provide the first mathematical runtime analysis of the NSGA-III, a refinement of the NSGA-II aimed at better handling more than two objectives. 
We prove that the NSGA-III with sufficiently many reference points -- a small constant factor more than the size of the Pareto front, as suggested for this algorithm --
computes the complete Pareto front of the 3-objective \textsc{OneMinMax} benchmark in an expected number of $O(n \log n)$ iterations.
This result holds for all population sizes (that are at least the size of the Pareto front). It shows a drastic advantage of the NSGA-III over the NSGA-II on this benchmark. The mathematical arguments used here and in previous work on the NSGA-II suggest that similar findings are likely for other benchmarks with three or more objectives.%\merk{put this in the paper?}
\end{abstract}

\section{Introduction}

Many practical applications require to optimize for multiple, conflicting objectives.
Such tasks can be tackled by population-based algorithms, whose population eventually represents a set of Pareto solutions, solutions that cannot strictly be dominated by any other solution.  
Thereby, they represent multiple useful trade-offs between the objectives and allow the user to choose among these according to their personal preferences.
Indeed, evolutionary algorithms (EAs), or, more precisely, multi-objective evolutionary algorithms (MOEAs), have been successfully applied to many real-world problems \cite{ZhouQLZSZ11}.
Among these, Zhou et al.~\cite{ZhouQLZSZ11} identify the non-dominated sorting genetic algorithm (\NSGAtwo) \cite{DebPAM02} as the most prominent one.

Both empirical evaluations \cite{KhareYD03,PurshouseF07} and recent mathematical runtime analyses \cite{ZhengLD22,DoerrQ22ppsn,BianQ22,DangOSS23,DoerrQ23LB,DoerrQ23crossover} confirm the strong results of the \NSGAtwo on bi-objective benchmarks.
The performance on problems with 3 or more objectives, however, is not as well understood.
Empirical studies, for example \cite{KhareYD03}, suggest that the \NSGAtwo struggles with such problems. A recent mathematical runtime analysis~\cite{ZhengD22arxivmany} shows that the \NSGAtwo regularly loses desirable solutions when optimizing the $3$-objective \threeOMM problem, and consequently, cannot find its Pareto front (the set of Pareto optimal solution values) in sub-exponential time.
As a remedy, Deb and Jain~\cite{DebJ14} proposed a modified version of the \NSGAtwo, called \NSGA.
It replaces the crowding distance, a measure which the \NSGAtwo uses in addition to the dominance relation to determine which individuals are taken in the next generation, by a procedure involving reference points in the solution space.
Their evaluations on benchmarks with 3 to 15 objectives show that the \NSGA is suitable for more than 2 objectives.

These empirical insights are, however, not yet backed with a theoretical understanding.
In order to fill this gap, we mathematically analyze the runtime of the \NSGA on the  \threeOMM problem. 
We show that by employing sufficiently many reference points (a small constant factor more than the size of the Pareto front, as suggested for this algorithm) and a population at least of the size of the Pareto front, $N \ge (\frac{n}{2}+1)^2$, once a solution for a point on the Pareto front is found, the population will always contain such a solution. This is a notable difference to the \NSGAtwo~\cite{ZhengD22arxivmany} and enables us to prove that after an expected number of $O(n \log n)$ iterations the \NSGA (for all future iterations) has a population that covers the Pareto front. Overall, this result indicates, in a rigorous manner, that the selection mechanism of the \NSGA has significant advantages over the one of the \NSGAtwo. Possibly, our result also indicates that more algorithm users should switch from the \NSGAtwo, still the dominant algorithm in practice, to the \NSGA. We note that the latter has as additional parameter the number of reference points, but the general recommendation to use by a small factor more reference points than the size of the Pareto front (or, in the case of approximate solutions, the size of the desired solution set) renders it easy to choose this parameter. We note that our results support this parameter choice, our proven guarantees also hold from the point on when the number of reference points is a small constant factor larger than the Pareto front. We note that using more reference points does not significantly increase the runtime (not at all when counting fitness evaluations and only moderately when counting wall-clock time), so in any case the choice of this parameter appears not too critical. 

\section{Previous Work}
For the sake of brevity, we do not further discuss empirical and practical works here.
Since the beginning of the century, mathematical runtime analyzes have been employed in order to gain insights and prove bounds on the running time of multi-objective randomized search heuristics \cite{LaumannsTZWD02,Giel03,Thierens03}.  
At first, research focused on analyzing simple, synthetic algorithms like the SEMO and the global SEMO (GSEMO).
Though in practical applications, usually more sophisticated algorithms are used, these analyses still led to useful insights.
Later, the runtimes of more realistic algorithms have been studied mathematically \cite{BrockhoffFN08,NguyenSN15,DoerrGN16,LiZZZ16,HuangZCH19,HuangZ20,BianZLQ23}.
Only recently, first mathematical runtime analyses of the \NSGAtwo on bi-objective benchmarks have appeared.
The first one of these proves a running time of $O(N n \log n)$ function evaluations on the \oneminmax benchmark and of $O(Nn^2)$ on the \lotz (\textsc{LeadingOnesTrailingZeroes}) benchmark, when employing a population of $N\ge 4(n+1)$ individuals \cite{ZhengLD22}.
A central observation in their proof is that this population size suffices to ensure that, once a solution for a point on the Pareto front is sampled, the population will always contain a solution with this objective value.
Employing a population size that exactly matches the size of the Pareto front does not suffice, as then, for an exponential time, the \NSGAtwo will miss a constant fraction of the Pareto front.
Nevertheless, a smaller population is still able to find good approximations of the Pareto front \cite{ZhengD22gecco}. 
Further, by assuming the objectives to be sorted identically, the required size of the population was reduced to $2(n+1)$ \cite{BianQ22}. 
The same work studies the \NSGAtwo when employing crossover, but does not improve the running time bounds of \cite{ZhengLD22}.
Also, it introduces a novel selection mechanism, improving the running time on the \lotz benchmark to $O(n^2)$.
Recently, the \NSGAtwo was studied on a multimodal benchmark \cite{DoerrQ23tec}. Very recently, also lower bounds were proven~\cite{DoerrQ23LB}, examples for the usefulness of crossover were found~\cite{DoerrQ23crossover,DangOSS23}, and a runtime analysis on a combinatorial optimization problem appeared \cite{CerfDHKW23}. 

Very recently, also lower bounds were proven and examples for the usefulness of crossover were found~\cite{DoerrQ23LB,DoerrQ23crossover,DangOSS23}. 
We note that all these results only cover bi-objective benchmarks.

The only mathematical runtime analysis of the \NSGAtwo on a benchmark consisting of more than two objectives gave a disillusioning result~\cite{ZhengD22arxivmany}. When run on the simple \oneminmax benchmark, a multi-objective version of the classic \onemax benchmark, and the number of objectives is some constant $m \ge 3$, the combined parent and offspring population can contain only $O(n)$ solutions with positive crowding distance. All other solutions have a crowding distance of zero, hence the selection between them is fully at random. Since this benchmark has a Pareto front size of order $n^{\lceil m/2 \rceil}$, at least such a population size is necessary when trying to compute the Pareto front. With these (necessary) parameters, almost all selection decisions are random, which easily implies that regularly Pareto optimal solution values are lost from the population. This easy argument suggests that the difficulties proven in that work are not restricted to the \oneminmax benchmark, but are likely to appear for many problems having at least three objectives.

For the simple SEMO, more results exist on benchmarks with more than two objectives. We describe them brief{}ly to ease the comparison, but we note that the very different main working principle -- keeping all solutions for which no dominating solution has been found -- makes this algorithm somewhat special and not always very practical. 
The first bounds on the expected number of function evaluations when optimizing the many-objective variants of \textsc{CountingOnesCountingZeroes} and \textsc{LeadingOnesTrailingZeroes}, \mcocz and \mlotz, are in $O(n^{m+1})$, for bit strings of length $n$ and $m$ objectives \cite{LaumannsTZ04}.
For \mcocz, the bound was later improved to $O(n^m)$, if $m>4$, and $O(n^3 \log n)$, if $m=4$ \cite{BianQT18ijcaigeneral}.
Further, the MOEA/D, an algorithm that decomposes a multi-objective problem into multiple single-objective problems which are then solved in a co-evolutionary manner, has been studied on \mcocz and \mlotz \cite{HuangZLL21}. 
As this algorithmic approach drastically differs from the one of the \NSGAtwo and \NSGA, we do not discuss these results in detail.

\section{Preliminaries}

We now define the required notation for multi-objective optimization as well as the considered objective functions and give an introduction to the \NSGA.

\subsection{Multi-objective Optimization}

For $m\in \N$, an \emph{$m$-objective function} $f$ is a tuple $(f_1,\ldots, f_m)$, where $f_i\colon \Omega\rightarrow \R$ for some search space~$\Omega$.
For all $x\in\Omega$, we define $f(x)=(f_1(x),\ldots, f_m(x))$.
Other than in single-objective optimization, there is usually no solution that minimizes all $m$ objective functions simultaneously.
For two solutions $x,y$, we say that $x$ \emph{dominates} $y$ and write $x\preceq y$ if and only if $f_j(x)\le f_j(y)$ for all $1\le j\le m$.
If additionally there is a $j_0$ such that $f_{j_0}(x)<f_{j_0}(y)$, we say that $x$ \emph{strictly dominates} $y$, denoted by $x\prec y$.
A solution is \emph{Pareto-optimal} if it is not strictly dominated by any other solution.
We refer to the set of objective values of Pareto-optimal solutions as the \emph{Pareto front}.
In our analyses, we analyze the number of function evaluations until the population covers the Pareto front, i.e., until for each value $p$ on the Pareto front the population contains a solution $x$ with $f(x)=p$.
For a vector $v=\colvec{v_1\\v_2\\v_3}$, we denote its length by 
\[|v|=\sqrt{v_1^2+v_2^2+v_3^2}.\]

\subsection{\threeOMM Benchmark}
We are interested in studying the \NSGA on a 3-objective function.
The \oneminmax function, first proposed by \cite{GielL10}, translates the well-established \onemax benchmark into a bi-objective setting. 
It is defined as $\oneminmax\colon \{0,1\}^n\rightarrow \N\times \N$ by 
\begin{align*}
    \oneminmax(x) &= (\zeromax(x),\onemax(x))\\
    &= \left(n-\sum_{i=1}^nx_i, \sum_{i=1}^nx_i\right)
\end{align*}
for all $x=(x_1,\ldots,x_n)\in\{0,1\}^n$.

We call its translation into a 3-objective setting \threeOMM (for 3-\oneminmax).
For even $n$, we define $\threeOMM\colon \{0,1\}^n\rightarrow \N^3$ by
\begin{align*}
 \threeOMM(x)
 %  = (\zeromax(x), \onemax(x_{1\ldots n/2}),
  %                 \onemax(x_{n/2+1\ldots n}))
   = \left(n-\sum_{i=1}^n x_i,
          \sum_{i=1}^{n/2} x_i,
          \sum_{i=n/2+1}^{n} x_i\right)
\end{align*}
for all $x=(x_1,\ldots,x_n)\in\{0,1\}^n$.

\subsection{NSGA-III}
The main structure of the \NSGA \cite{DebJ14} is identical to the one of the \NSGAtwo \cite{DebPAM02}.
It is initialized with a random population of size $N$.
In each iteration, the user applies mutation and/or crossover operators to generate an offspring population of size $N$.
As the NSGA framework is an MOEA with a fixed population size, out of this total of $2N$ individuals, $N$ have to be selected for the next iteration.

Because non-dominated solutions are to be preferred, the following ranking scheme is used to set the dominance relation as the predominant criterion for the survival of individuals. 
Individuals that are not strictly dominated by any other individual in the population obtain rank 1.
Recursively, the other ranks are defined. Each individual that has not yet been ranked and is only strictly dominated by individuals of rank \(1,\ldots,k-1\) is assigned rank $k$. 
Clearly, an individual is more interesting the lower its rank is. 
Let $F_i$ denote the set of individuals with rank $i$ and let $i^*$ be minimal such that $\sum_{i=1}^{i^*} |F_i| \ge N$.
All individuals with rank at most $i^*-1$ survive into the next generation.
Further, $0<k\le N$ individuals of rank $i^*$ have to be selected for the next generation such that the new population is again of size $N$, and the next iteration can begin. 
The only difference between the \NSGAtwo and the \NSGA is the procedure of selecting the $k$ individuals of rank $i^*$.
While the \NSGAtwo employs crowding-distance, the \NSGA uses reference points, typically distributed in some structured manner on the normalized hyperplane, in order to select a diverse population.
For the whole framework, see Algorithm~\ref{alg:nsga3}. 
Note that whenever we refer to sets of individuals, we are actually referring to multi-sets as each solution might be represented multiple times in the population.
\begin{algorithm2e}[t]%
Let the initial population $P_0$ be composed of $N$ individuals chosen independently and uniformly at random from $\{0,1\}^n$.

\For{$t = 0, 1, 2, \ldots$}{
Generate offspring population $Q_t$ with size $N$

Use fast-non-dominated-sort() from Deb et al.\ \cite{DebPAM02}) to divide $R_t = P_t \cup Q_t$ into $F_1, F_2, \ldots$

Find $i^* \ge 1$ such that $\sum_{i=1}^{i^*-1} |F_i| < N$ and $\sum_{i=1}^{i^*} |F_i| \ge N$

$Z_t = \bigcup_{i=1}^{i^*-1}F_i$

Select $\Tilde{F_{i^*}}\subseteq F_{i^*}$ such that $|Z_t\cup\Tilde{F_{i^*}}| = N$ (use crowding-distance for \NSGAtwo and Algorithm~\ref{alg:selection} for \NSGA)

$P_{t+1} = Z_t \cup \Tilde{F_{i^*}}$
}
\caption{NSGA-II and NSGA-III}
\label{alg:nsga3}
\end{algorithm2e}%

In order to select individuals from the critical rank~$i^*$, the \NSGA normalizes the objective functions and associates each individual with a reference point.

Regarding the normalization step, we do not consider the procedure as given in \cite{DebJ14} but the improved and more detailed normalization given in \cite{Blank_Deb_Roy_2019} by one of the two original authors among others.
Consider any iteration.
Let $\hat{z}_j^*$ be the minimum observed value in the $j$th objective over all generations including the current offspring. 
We use $\hat{z}_j^*$ to estimate the ideal point.
Further, for each objective $j$ we compute an extreme point in that objective by using a achievement scalarization function. 
Consider the hyperplane spanned by these points.  
The intercepts of this hyper plane with the coordinate axes give the Nadir point estimate $\hat{z}^{\nad}$.
In case that $H$ is not well-defined by the extreme points or if an intercept is either smaller than a given positive threshold $\epsilon_{\nad}$ or larger than the highest observed value over all generations in that objective, $\hat{z}^{\nad}$ is instead defined by the maximum value in each objective of individuals in the first non-dominated front.
Last, for each objective in which the Nadir point estimate is smaller than the ideal point estimate plus the threshold $\epsilon_{\nad}$, the maximum value in that objective over all current non-dominated fronts is used for the Nadir point estimate in that objective instead.
The normalized objective functions $f^n$ are now defined as 
 \begin{equation}\label{eq:normalization}
     f_j^n(x) = \frac{f_j(x)-\hat{z}_j^*}{\hat{z}_j^{\nad}-\hat{z}_j^*}
 \end{equation} 
 for $j\in \{1,\ldots,M\}$.
Algorithm~\ref{alg:normalize} formalizes the normalization procedure, for a more detailed description see \cite{Blank_Deb_Roy_2019}. We note that Blank et al.\ span the hyperplane $H$ by the extreme points \emph{after} subtracting the ideal point estimate $\hat{z}^*$, while in the interest of a clear notation we span $H$ by the original extreme points. This leads some individual lines slightly differing from  \cite{Blank_Deb_Roy_2019}, though the described algorithm is identical.
\begin{algorithm2e}[ht]%
\Input{%
$f=(f_1,\ldots, f_M)$: objective function\\
$z^*\in \R^M$: observed min.\ in each objective\\
$z^w\in \R^M$: observed max.\ in each\\ objective\\
$E\subseteq \R^M$: extreme points of previous\\iteration, initially $\{\infty\}^M$}
%\Output{$f^n=(f_1^n,\ldots, f_M^n)$: normalized objective\\ function}
\For{$j=1$ \KwTo $M$}{
    $\hat{z}_j^* = \min \{z_j^*, \min_{z\in Z} f_j(z)\}$\\
    Determine an extreme point $e^{(j)}$ in the $j$th objective from $Z\cup E$ using an achievement scalarization function
}
$\valid = \mFalse$\\
\If{$e^{(1)},\ldots, e^{(M)}$ \emph{are linearly independent}}{
    $\valid = \mTrue$\\
    Let $H$ be the hyperplane spanned by $e^{(1)},\ldots, e^{(M)}$\\  
    \For{$j=1$ \KwTo $M$}{
        $I_j = $ the intercept of $H$ with the $j$th objective axis\\
        \If{$I_j < \epsilon_{\nad}$ \emph{or} $I_j > z^w_j$}{
            $\hat{z}_j^{\nad} = I_j$
        } \Else{
           $valid = \mFalse$\\
          \Break
        }
    }
} 
\If{$valid = \mFalse$}{ 
    \lFor{$j'=1$ \KwTo $M$}{$\hat{z}_j^{\nad} = \max_{x\in F_1} f_j(x)$}
}
\For{$j=1$ \KwTo $M$}{
    \lIf{$\hat{z}_j^{\nad} < \hat{z}_j^* + \epsilon_{\nad}$}
       {$\hat{z}_j^{\nad} = \max_{x\in F_1\cup\ldots\cup F_k} f_j(x)$}
}
Define $f^n_j(x) = (f_j(x)-\hat{z}_j^*) / (\hat{z}_j^{\nad}-\hat{z}_j^*) \quad \forall x\in\{0,1\}^n, j\in\{1,\ldots,M\}$
\caption{Normalization as given by [Blank et al., 2019]}
\label{alg:normalize}
\end{algorithm2e}%

After the normalization, each individual of rank at most $i^*$ is associated with its closest reference point with respect to the normalized objectives. 
More precisely, an individual $x$ is associated to the reference point $\refp(x)$ that minimizes the angle between the point vectors of $x$ and $\refp(x)$. That is, $\refp(x)$ is the reference point such that the distance between $x$ and the line passing through the origin and $\refp(x)$ is minimal.
%\simon{What about ties? Not clear in \cite{DebJ14}}
Then, one iterates through the reference points, always selecting the one with the fewest associated individuals that are already selected for the next generation. Ties are resolved randomly.
If the reference point only has associated individuals that are already selected, it is skipped.
Otherwise, among the not yet selected individuals the one closest to the reference point (with respect to the normalized objective function) is selected for the next generation. 
Once more, ties are resolved randomly.
If the next generation already contains an individual that is associated with the reference point, other measures than the distance to the reference point can be considered.
The selection terminates as soon as the required number of individuals is reached.
This procedure is formalized in Algorithm~\ref{alg:selection}.
\begin{algorithm2e}[tbh]%

\Input{$Z_t$: the multi-set of already selected individuals\newline
$F_{i^*}$: the multi-set of individuals to choose from}
%\Output{$\Tilde{F_{i^*}}$ with $|Z_t \cup \Tilde{F_{i^*}}| = N$}

$f^n = \textsc{Normalize}(f, Z=Z_t\cup F_{i^*})$ using Algorithm~\ref{alg:normalize}

Associate each individual $x\in Z_t \cup F_{i^*}$ to the reference point $\refp(x)$ 

For each reference point $r\in R$, let $\rho_r$ denote the number of (already selected) individuals in $Z_t$ associated with $r$

$R' = R$

$\Tilde{F_{i^*}} = \emptyset$

\While{$\mTrue$}{
    Let $r_{\min}\in R'$ be such that $\rho_{r_{\min}}$ is minimal (break ties randomly) 
    
    Let $x_{r_{\min}}\in F_{i^*} \setminus  \Tilde{F_{i^*}}$ be the individual that is associated with $r_{\min}$ and minimizes the distance between $f^n(x_{r_{\min}})$ and $r_{\min}$ (break ties randomly)\footnotemark
    
    \If{$x_{r_{\min}}$ \emph{exists}}{
    
        $\Tilde{F_{i^*}} = \Tilde{F_{i^*}} \cup \{x_{r_{\min}}\}$

        $\rho_{r_{\min}} = \rho_{r_{\min}} +1 $
    
        \If{$|S_t|+|\Tilde{F_{i^*}}|=N$}{
            \Breakall and \Return $\Tilde{F_{i^*}}$
        }
    }
    \Else{
        $R' = R' \setminus \{r\}$
    }
}
\caption{Selection based on a set $R$ of reference points when maximizing the function $f$}
\label{alg:selection}
\end{algorithm2e}%
\footnotetext{If $\rho_{r_{\min}}>0$, $x_{r_{\min}}$ can be selected in any other diversity-preserving manner from the associated individuals.}
\FloatBarrier

For our analyses, we assume that the \NSGA employs a set of structured reference points in the normalized hyperplane as proposed by Deb and Jain \cite{DebJ14}.
In the case of 3 objectives, this corresponds to a set of points in the triangle spanned by $\colvec{1\\0\\0}, \colvec{0\\1\\0},$ and $\colvec{0\\0\\0}$.
Divide the lines between two pairs of these points into $p$ divisions of equal length.
Consider the lines that pass through the start and end points of all divisions and are orthogonal to the respective side. 
Then, a reference point is placed at every intersection of these lines, see Figure~\ref{fig:referencePoints}.
By \cite[Equation~3]{DebJ14}, this creates $\binom{3+p-1}{p}=\binom{p+2}{2}$ reference points.
Observe that these reference points partition the non-negative domain of the spanned triangle in regular hexagonal Voronoi cells.

\begin{figure}
    \centering
    \includegraphics[width=.45\textwidth]{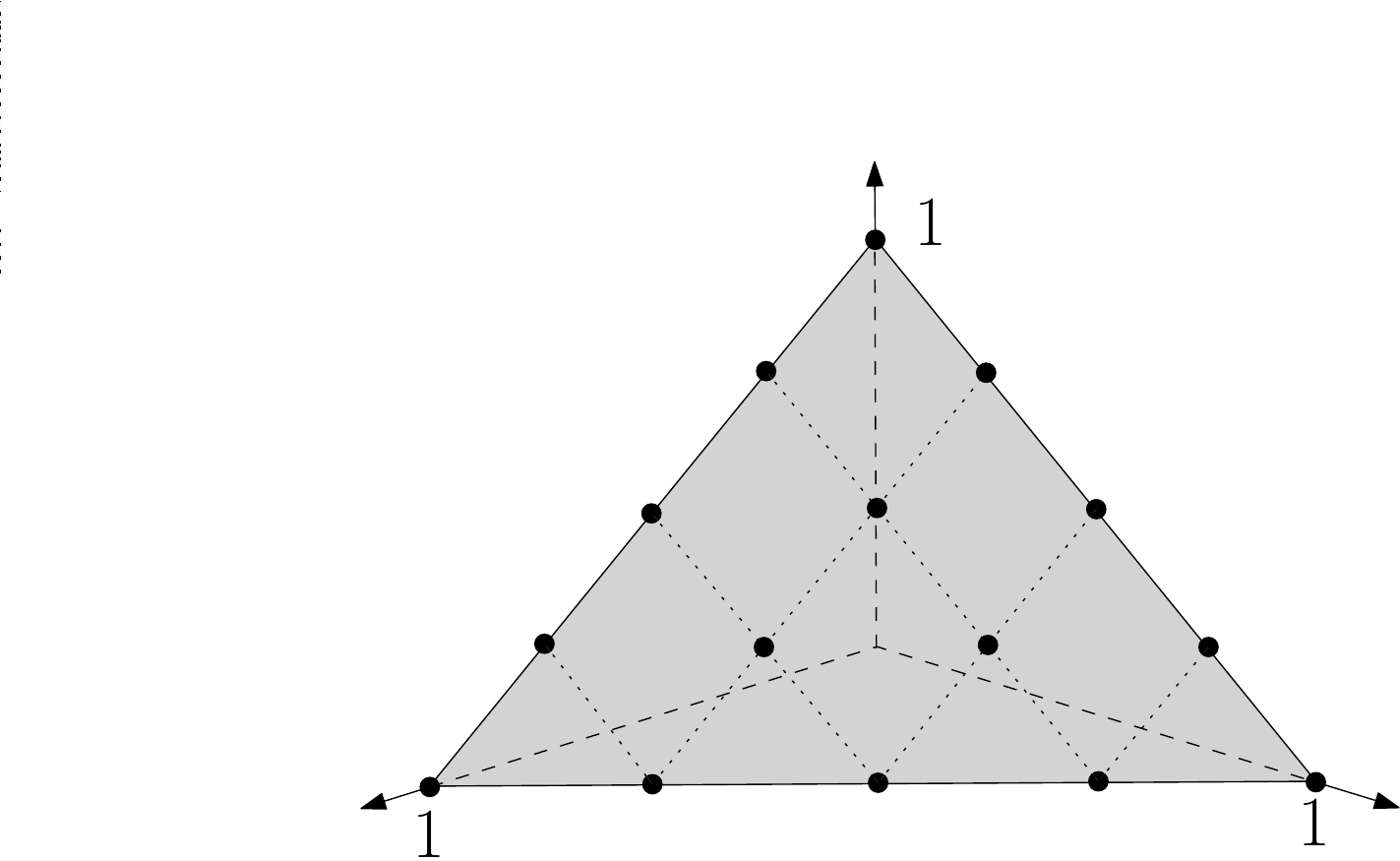}
    \caption{Structured set of reference points for 3 objectives with $p=4$ divisions based on [Deb and Jain, 2014, Figure~1].}
    \label{fig:referencePoints}
\end{figure}

\section{How the Reference Point Mechanism Avoids Losing Solution Values}

Before analyzing the optimization time of the \NSGA on \threeOMM, we show that, by employing sufficiently many reference points, once the population covers a point on the Pareto front, it is covered for all future iterations.
To this end, we first analyze how the normalization shifts the points on the Pareto front to then conclude that every reference point is associated with at most one point of the Pareto front.
With this, we argue that already sampled points on the Pareto front are never lost again. 

Our analysis assumes that the population is non-degenerated, i.e., that for each objective the maximum value over all generations is larger than the minimum value over all generations. 
However, w.h.p. this holds starting in the first iteration.
We note that for \threeOMM no individual strictly dominates another. Thus, all individuals are in the first rank and we analyze on the way ties inside that rank are broken, focusing on the only difference between the \NSGAtwo and the \NSGA.

\begin{lemma}\label{lem:normalization}
    Assume that for each objective value observed in a previous iteration, the current population contains at least one individual with that objective value. 
    Let for each objective $i$, $z_i^{\min}$ and $z_i^{\max}$ be the minimum and maximum value in the current population.
    Then, each objective $i$ is normalized as $f^n_i(x) = \frac{f_i(x)-z_i^{\min}}{z_i^{\max} - z_i^{\min}}$.
\end{lemma}
\begin{proof}
   Note that due to the assumption in the statement the $z_i^{\min}$ and $z_i^{\max}$ are the extreme values not only for the current population but also among all previous generations.
   Thus, $\colvec{z_1^{\min}\\z_2^{\min}\\z_3^{\min}}$ is exactly the ideal point estimate $\hat{z}^*$.
   Further, $\colvec{z_1^{\max}\\z_2^{\max}\\z_3^{\max}}$ is the Nadir-point estimate $\hat{z}^{\nad}$.
   This estimate is defined by the intercepts of the hyperplane spanned by the extreme points with the objective axes though its objective values might be set to the maximum of the respective objective among all individuals or all non-dominated individuals in the current generation.
   As for \threeOMM, all individuals are non-dominated, in the latter case the Nadir point estimate in each objective $j$ is exactly $z_j^{\max}$. 
   The first case only occurs if the extreme points describe a well-defined hyperplane and the intercepts of that hyperplane are at most the maximum observed value in the respective objective over all generations.
   Note that all possible objective values for \threeOMM lay in the same plane $E: v_1+v_2+v_3 = n$.
   The intercept of this plane with each objective is $n$.
   As for \threeOMM, the second and third objective never take any value larger than $n/2$, this case never occurs.
   Thus, by Equation~\ref{eq:normalization}, $f^n_i(x) = (f_i(x)-z_i^{\min})/(z_i^{\max} - z_i^{\min})$.
\end{proof}

\begin{lemma} \label{lem:uniqueAssociations}
    By employing at least $p=21n$ divisions along each objective, all
individuals that are associated with the same reference point have the same
objective value.
\end{lemma}

We note that we have not tried to optimize the constant $21$ as the number of reference points does not have a large influence on the runtime. 

\begin{proof}
    Each individual $x$ is associated to the reference point $r$ that minimizes the distance between $f^n(x)$ and the line between the origin and $r$.
    Equivalently, $x$ is associated to the reference point $r$ that minimizes the angle between the point vectors $r$ and $f^n(x)$. 
    To show the statement, we first upper bound the angle between normalized objective values and their nearest reference point and then lower bound the angle between two different normalized objective values to show that for $p \ge 21n$ the latter is more than twice as large as the former.
    Thereby, individuals with different objective values are never associated with the same reference point because if they were then the angle between their normalized objective values would be at most the sum of their angles to that reference point.

    To upper bound the angle between the vectors of $f^n(x)$ and $r$, first note that when scaling $f^n(x)$ to $t=a\cdot f^n(x)$ such that $t$ lies in the non-negative domain of the reference point plane ($\sum_{i=1}^3 t_i =1$), the angle between $t$ and $r$ is the same as between $f^n(x)$ and $r$.
    When employing $p$ divisions, the reference points partition the non-negative domain of the reference point plane into equilateral triangles with side length $\sqrt{2}/p$.
    As $t$ lies in one of these triangles (including the boundary), there is a reference point $r$, namely a corner of the triangle,  such that $|t-r|\le \sqrt{2}/p$.

    Consider the point $z=(t+r)/2$ exactly in between $t$ and $r$ and $\Delta = t-z$.
    Then, $t = z+\Delta$ and $r=z-\Delta$.
    Thus, the angle between the point vectors $t$ and $r$ is at most $\cos^{-1}(\frac{(z+\Delta)\circ (z-\Delta)}{|(z+\Delta)|\cdot |(z-\Delta)|})$.
    As $|z|^2 \ge 1/9$ %\simon{Tatsächlich sogar $|z|^2 \ge 1/3$, wenn ich mich gerade nicht irre. In der eingereichten Version steht aber leider noch 1/9. Wie machen wir das? (1/9 gilt ja zum Glück auch)} 
    because $\sum_{i=1}^3 z_i=1$ and $|\Delta| \le \sqrt{2}/p$, we have
    \begin{align*}
       &\frac{(z+\Delta)\circ (z-\Delta)}{|(z+\Delta)|\cdot |(z-\Delta)|}\\
     &\quad = \frac{\sum_{i=1}^3 (z_i+\Delta_i)(z_i-\Delta_i)}{\sqrt{\sum_{i=1}^3 (z_i+\Delta_i)^2}\sqrt{\sum_{i=1}^3 (z_i-\Delta_i)^2}}\\
     &\quad \ge \frac{\sum_{i=1}^3 z_i^2 - \Delta_i^2}{\sum_{i=1}^3 z_i^2 }\\
     &\quad = \frac{|z|^2 - |\Delta|^2}{|z|^2}
     = 1 - \frac{|\Delta|^2}{|z|^2}
     \ge 1 - \frac{18}{p^2}.
    \end{align*}
    Since the function $\cos^{-1}$ is decreasing, this yields that the angle between the vectors is at most $\cos^{-1}(1-18/p^2)$.

    Next, consider any two individuals objective values $x,y$ with $f(x)\neq f(y)$.
    Let $\zmax$ and $\zmin$ be the points having the maximum and minimum value in each objective across the current population, respectively.
    Then, by Lemma~1, for each objective $i$
    \begin{align*}
        f^n_i(x)&= \frac{f_i(x)-\zmin_i}{\zmax_i-\zmin_i}
        \intertext{and}
        f^n_i(y)&= \frac{f_i(y)-\zmin_i}{\zmax_i-\zmin_i}.
    \end{align*}
    Consider the point  $z = (f^n(x) + f^n(y))/2$ exactly in between $f^n(x)$ and $f^n(y)$ and let $\Delta = f^n(x) - z$.
    Then, $f^n(x) = z+ \Delta$ and $f^n(y) = z-\Delta$.
    Thus, the angle between the point vectors of $f^n(x)$ and $f^n(y)$ is $\cos^{-1}(\frac{(z+\Delta)\circ (z-\Delta)}{|(z+\Delta)|\cdot |(z-\Delta)|})$. We have
\begin{align*}
   & \frac{(z+\Delta)\circ (z-\Delta)}{|(z+\Delta)|\cdot |(z-\Delta)|}\\
  %      &= \frac{|z|^2 - |\Delta|^2}{\sqrt{\sum_{i=1}^2 z_i^3 + \Delta_i^2 + 2z_i\Delta_i}\sqrt{\sum_{i=1}^3 z_i^2 + \Delta_i^2 - 2z_i\Delta_i}}\\
   & = \frac{|z|^2 - |\Delta|^2}{\sqrt{|z|^2+|\Delta|^2+2\sum_{i=1}^3 z_i\Delta_i}}\\
    &\quad \cdot \frac{1}{\sqrt{|z|^2+|\Delta|^2-2\sum_{i=1}^2 z_i\Delta_i}}\\
   & = \frac{|z|^2 - |\Delta|^2}{\sqrt{(|z|^2+|\Delta|^2)^2-(2\sum_{i=1}^3 z_i\Delta_i)^2}}.
  %$  &= \frac{|z|^2 - |\Delta|^2}{\sqrt{(|z|^2+|\Delta|^2)^2-(2 z \circ \Delta)^2}}.
\end{align*}
The critical cases occur if $f^n(x)$ and $f^n(y)$ are as close to another as possible as this gives the smallest angle. 
As $x$ and $y$ are not comparable, this happens if the objective values differ by exactly 1 in two objectives before normalization.
After normalization, a difference of 1 in any objective corresponds to a difference of at least $1/n$ in that objective. 
Thus, the extreme case we have to consider is $\Delta = \colvec{1/(2n)\\-1/(2n)\\0}$ and its variants like $\Delta = \colvec{1/(2n)\\0\\-1/(2n)}$ or $\Delta = \colvec{-1/(2n)\\1/(2n)\\0}$.
By symmetry, it suffices in the following to discuss the case $\Delta = \colvec{1/(2n)\\-1/(2n)\\0}$.
In this case, we have $|\Delta|^2=1/(2n^2)$, so the cosinus of the angle is at most 
\begin{align*}
    & \frac{|z|^2 - |\Delta|^2}{\sqrt{(|z|^2+\frac{1}{2n^2})^2-(2(\frac{z_1}{2n}-\frac{z_2}{2n}))^2}}\\
    &= \frac{|z|^2 - |\Delta|^2}{\sqrt{|z|^4+\frac{|z|^2}{n^2}+\frac{1}{4n^4}-(\frac{z_1}{n}-\frac{z_2}{n})^2}}\\
    &\le \frac{|z|^2 - |\Delta|^2}{\sqrt{|z|^4+\frac{|z|^2}{n^2}-\frac{z_1^2}{n^2}}}\\
    &\le \frac{|z|^2 - |\Delta|^2}{|z|^2}
     = 1- \frac{|\Delta|^2}{|z|^2} 
     = 1- \frac{1}{6n^2},
\end{align*}
 because $z_1^2 \le |z|^2$ and $z$ lies in the unit cube so $|z|\le \sqrt{3}$.

It remains to determine $p$ such that $\cos^{-1}(1- 1/(6n^2)) > 2 \cos^{-1}(1- 18/p^2)$.
If $n=2$ and $p \ge 21n$ , we have 
\begin{align*}
    \cos^{-1}(1- 1/(6n^2)) = \cos^{-1}(1- 1/24)\\
    > 2 \cos^{-1}(1- 18/42^2) \ge  2 \cos^{-1}(1- 18/p^2).
\end{align*}
For $n > 2$, we note that the slope of the $\cos^{-1}$ function at $1- 1/(6n^2)$ only gets steeper, so $p \ge 21n$  suffices for these cases as well.
Thus, when using $p \ge 21n$ divisions for the reference points, the angle between the normalized objective values of any pair of individuals is more than twice as large as the angle between any individual and its closest reference point, when considering their respective vectors from the origin.
Thus, no two individuals with different objective values are associated with the same reference point.
\end{proof}

We note that $21n$ divisions correspond to $\binom{21n+2}{2} \in \mathcal{O}(n^2)$ divisions, so the number of reference points differs by the size of the Pareto front by a constant factor.

Our bound of $p\ge 21n$ is likely not tight in order to guarantee unique associations, as suggested by our experiments in Section~\ref{sec:experiments}.
Nevertheless, it is only off by a constant factor to the actual bound. 
This holds as we require more than $n/\sqrt{2}$ divisions to create $(n/2 + 1)^2$ reference points, one for each possible objective value.  
If we do not have that many reference points, there are at least two different objective values associated with the same reference point and when selecting individuals from that reference point by chance no individuals of one of the objective values might survive.

Using Lemma~\ref{lem:uniqueAssociations}, we are now able to show that the population does not forget an objective value once it sampled a solution for it.

\begin{lemma}\label{lem:notLooseSolution}
Consider the \NSGA optimizing a multi-objective function $f$ with a population of size~$N$. 
Let $F^*$ be the Pareto front of $f$.
Assume that in each iteration the number of objective values of non-strictly dominated individuals is at most $N$ and that all individuals associated with the same reference point have the same objective value.
Then, once the population contains a solution $x$ with $f(x)\in F^*$, the population will contain a solution for $x'$ such that $f(x')=f(x)$ in all future iterations.
\end{lemma}
\begin{proof}
Consider any iteration with a population that contains such an $x'$.
After the recombination and mutation step, the complete population of old solutions and offsprings contains $2N$ solutions.
Let $F_1$ denote the subset of solutions that are not strictly dominated.
Then, $x'\in F_1$.
If $|F_1| \le N$, all individuals in $F_1$ including $x'$ survive into the next generation.
Otherwise, the objective functions are normalized with respect to $F_1$ and the individuals in $F_1$ are associated with a reference point each.
By our assumptions, there are now at most $N$ reference points with at least one associated individual.
Thus, at least one individual is selected from each reference point with non-empty association set.
In particular, one of the individuals associated with the same reference point as $x'$ survives.
By our assumption, it has the same objective value as $x'$.
\end{proof}

\section{Runtime of the NSGA-III on \threeOMM}
We are now able to give a first upper bound on the expected optimization time of the \NSGA on \threeOMM. 
%
%We assume the recombination and mutation step to be such that each individual in the population has at least a constant chance to produce an offspring via standard mutation. %with a Hamming distance of 1 to any individual in the population is created with a probability in $\Omega(\frac{1}{n})$. 
%For example, this is achieved if each individual in the population has at least a constant probability to produce an offspring by applying standard uniform mutation in each iteration.

\begin{theorem}\label{thm:runTime3OMM}
Consider the \NSGA with population size $N \ge (\frac{n}{2}+1)^2$ and $p\ge 21n$ divisions along each objective optimizing \threeOMM. 
Assume the reproduction step to be such that each individual in the population has at least a chance $c^{-1}$ to create an offspring via standard mutation. 
Then, w.h.p., the population covers the Pareto front after $4ecn \ln(n)$ iterations, which corresponds to $O(cn^3 \log (n))$ evaluations of the fitness function.
\end{theorem}
\begin{proof}
We upper bound the probability that after $4cen\ln (n)$ iterations not all objective values on the Pareto front have been sampled.
To this end, we first give an upper bound on the probability that any specific objective value $(a,b)$ has not been sampled after $4cen\ln (n)$ iterations.
In each iteration $i$, let $d_i = \min_{s\in \text{Population}} |f_1(s)-a| + |f_2(s)-b|$ be the distance of $(a,b)$ to the closest individual in the population, i.e., the minimum number of bitflips required to turn this individual into one with objective value $(a,b)$.
By Lemma~\ref{lem:notLooseSolution} the population will never lose a sampled objective value, so $d_i$ never decreases.
Further, for all  $1\le \ell \le n$, define the geometrically distributed random variable $X_\ell$ as the number of iterations $i$ with $d_i = \ell$.
Then, the number of iterations until $(a,b)$ is covered is $X = \sum_{\ell = 1}^n X_\ell$.
Observe that $X_\ell$ has a success probability of at least $p_i=\frac{i}{ecn}$ by choosing the closest individual for mutation ($\frac{1}{c}$), flipping any of the at least $i$ bits that take it closer ($\frac{i}{n}$), and not flipping any other bit ($\frac{1}{e}$).
By \cite[Theorem~1.10.35]{Doerr20bookchapter}, 
    $\Pr[X\ge (1+3)ec n \ln (n)] \le n^{-3}$, 
i.e., the probability that the population does not cover $(a,b)$ after $4ec n \ln (n)$ iterations is at most $n^{-3}$.
Hence, by the union bound, the probability that all $(\frac{n}{2}+1)^2$ objective values are sampled after $4ecn \ln (n)$ iterations is at least 
\begin{align*}
    1-\left(\frac{n}{2}+1\right)^2 \cdot n^{-3}
\ge 1- \frac{1}{n}.
\end{align*}
Thus, w.h.p.\ $4ec\ln (n)$ iterations suffice to cover the complete Pareto front. 
Each of these iterations employs $(\frac{n}{2}+1)^2$ fitness evaluations, so w.h.p. a total of $O(cn^3\ln (n))$ fitness evaluations are required.
\end{proof}

\section{Experimental Evaluation}\label{sec:experiments}
We support our theoretical findings with empirical evaluations of the \NSGA on the \threeOMM benchmark. 
To this end, we employ the DEAP library \cite{DEAP_JMLR2012}, which provides a framework for benchmarking evolutionary algorithms and holds an implementation of the \NSGA and \NSGAtwo.\footnote{See \url{https://github.com/SimonWiet/experiments_nsga3} for code and results of the experiments. We slightly modified the implementation of the \NSGAtwo as when computing the crowding distance and selecting individuals, the library sorts on the unmodified list of individuals. As it employs a stable sorting, this results in some bias on the order of individuals with same objective values or crowding distance. This bias resulted in significantly different performances on the \threeOMM benchmark. In our implementation, the respective lists are additionally shuffled in order to not distort the performance by these effects and stay as close to the intuitive interpretation of the original algorithm as possible.}
All experiments on the \NSGA are conducted with a population size of $N=(n/2+1)^2$ (the size of the complete Pareto front) and $p=4.65n$ divisions on the reference point plane.
While Theorem~\ref{thm:runTime3OMM} requires $p\ge 21n$, the theoretical proven bound is most likely not tight. 
In all our experiments we included a control routine to check whether any points on the Pareto front were lost during the process.
For $p=4.65n$, no point was ever lost, suggesting a much smaller bound than $p \ge 21n$. 
In the reproduction step, we apply standard bit mutation (flipping each bit independently with probability $1/n$) on each individual.

Figure~\ref{fig:coverages} shows the coverage of the Pareto front of \threeOMM in runs of the \NSGA and the \NSGAtwo. While all 3 runs of the \NSGA (with a population size equal to the size of the Pareto front) find the complete Pareto front in less than 300 iterations, the \NSGAtwo shows a real progress only in the first few iterations and then stagnates at some fraction of the Pareto front covered (with some variance).
Increasing the population size mildly increases this stagnation level, but even with a population of 8 times the size of the Pareto front, the \NSGA never has even 300 out of the 441 elements in the Pareto front. The figure shows that the effect of doubling the population size of the \NSGAtwo reduces with increasing population size, which suggests that a truly large population size would be needed in order to find the complete Pareto front in polynomial time.

The figure illustrates well that the advantage of the \NSGA over the \NSGAtwo lies in its property of never losing Pareto dominant solutions, which is reflected in the non-decreasing curve of the \NSGA in the figure as opposed to the oscillating curves of the \NSGAtwo. 

\begin{figure}
    \centering
    \includegraphics[width=\textwidth]{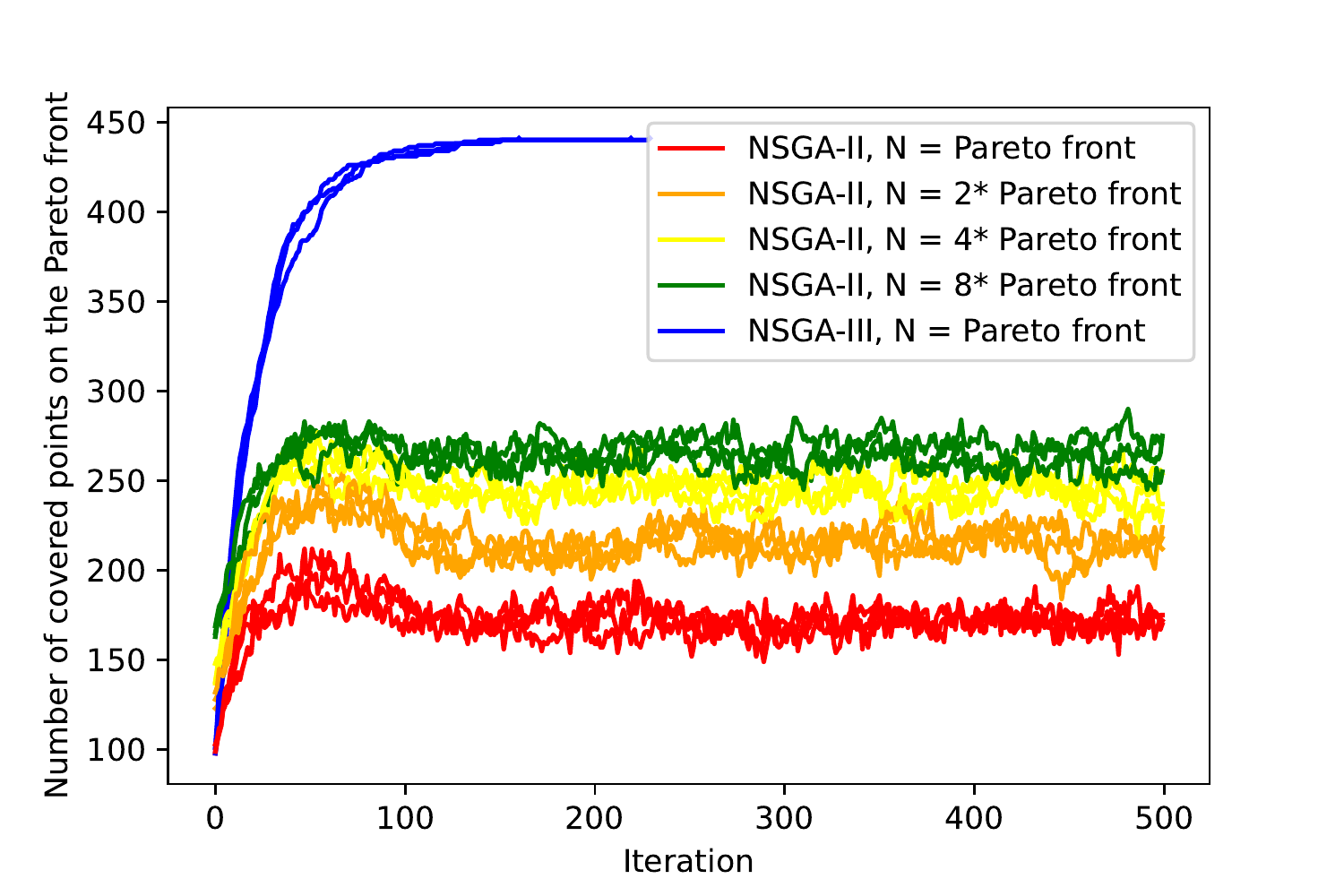}
    \caption{The number of covered points on the Pareto front per iteration of the NSGA-III and the NSGA-II with different population sizes based on the size of the Pareto front (441) when optimizing the 3-OMM benchmark. Three independent runs are depicted for each algorithm/population size. All algorithms only use mutation for the reproduction step. The lines for the NSGA-III (blue) stop after the first iteration in which the the complete Pareto front is covered.}
    \label{fig:coverages}
\end{figure}

Further, we run the NSGA-III on \threeOMM using not only standard bit mutation but also uniform crossover (for a pair of parent bit strings, at each bit position there is a chance of 1/2 of swapping the respective bits between the two parents). We randomly partitioned the population in pairs of 2 bit strings and for each pair tested a chance of $1/2$ or $9/10$ to perform a crossover on this pair before applying standard mutation.
Figure~\ref{fig:optimizationTimes} shows the number of iterations it took the \mbox{\NSGA} to find the complete Pareto front of the \mbox{\threeOMM} benchmark for the different crossover rates and over different values of the problem size $n$.
The data shows quite clearly that a crossover rate of $1/2$ moderately increases the runtime and that a crossover rate of $0.9$ significantly increases the runtime. 
This suggests that the main progress is due to the mutation operator, justifying our focus on standard bit mutation in our analysis.

\begin{figure}
    \centering
    \includegraphics[width=\textwidth]{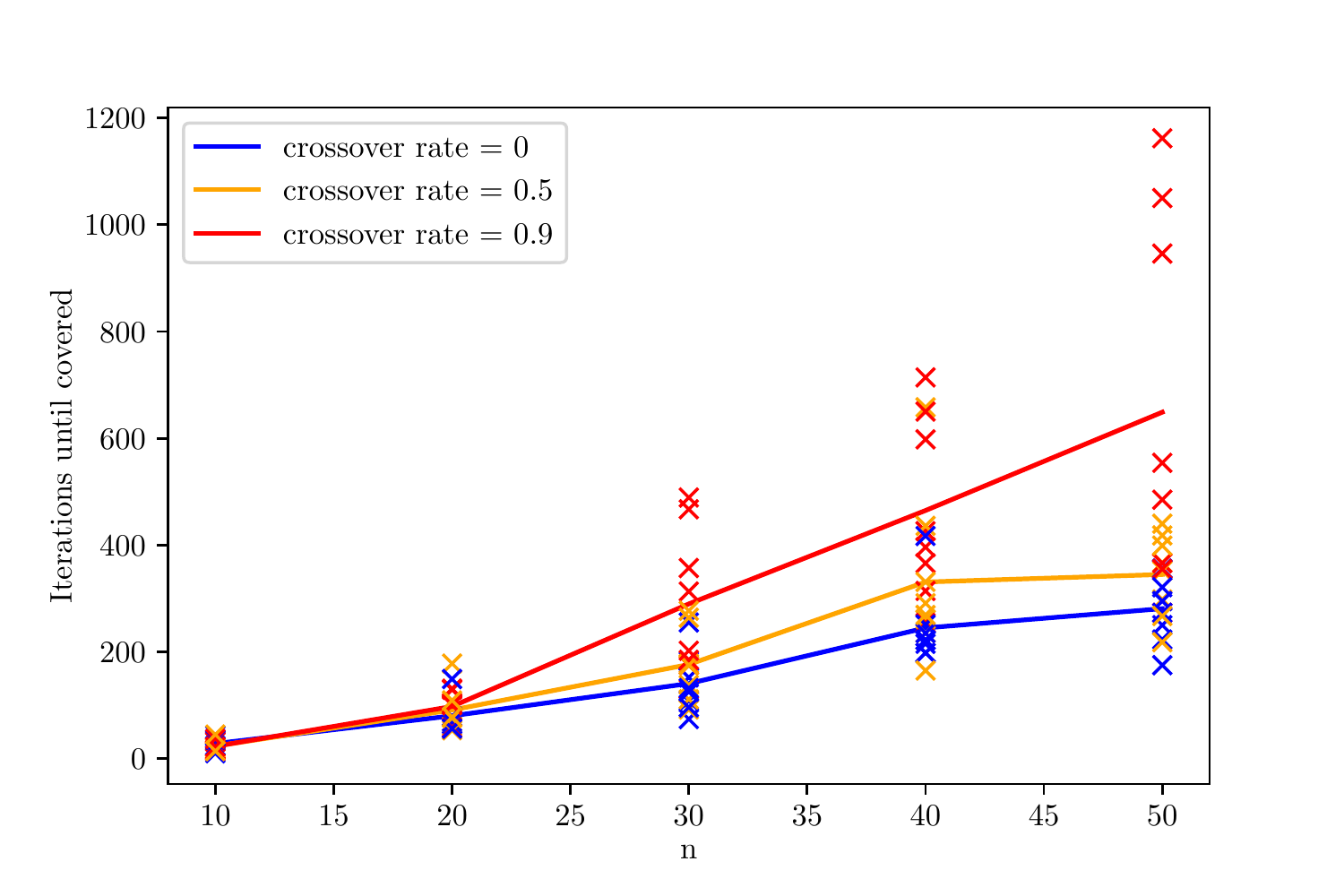}
    \caption{The number of iterations required to cover the complete Pareto front for multiple runs of the NSGA-III (8 runs for each of the different values of $n$ and the different crossover variants (mutation only, and crossover rates of $1/2$ and $9/10$.}
    \label{fig:optimizationTimes}
\end{figure}
\FloatBarrier

\section{Conclusion}

In this first mathematical runtime analysis of the \NSGA, we proved that this algorithm with suitable parameters never loses a Pareto optimal solution value  when optimizing the \threeOMM problem. This is in drastic contrast to the \NSGAtwo, which provably also with large populations regularly suffers from such losses and consequently cannot find (or approximate beyond a certain constant factor) the Pareto front of \threeOMM in sub-exponential time~\cite{ZhengD22arxivmany}. From this positive property of the \NSGA, we derive that it finds the full Pareto front of \threeOMM in an expected number of $O(n \log n)$ iterations. Our experimental results confirm the drastically different behavior of the algorithms.

From the mathematical proofs we are optimistic that the key property of not losing desired solution values extends to broader classes of problems. In fact, our main argument was that the angle (with the origin) between any two different normalized solutions is large enough to prevent that both are associated with the same reference point. We have not used any properties of the optimization process there, but only the structure of the set of objective values of the problem. We are thus optimistic that similar properties hold for other optimization problems, say with integral objectives of a bounded range. 

This one first rigorous result on the \NSGA clearly is not enough to derive reliable advice on which MOEAs to prefer, but we believe that it should motivate more users of the \NSGAtwo to also try the \NSGA. If only part of the drastic differences proven here extend to broader classes of problems, this switch is definitely justified. 

}%sloppy, please do not remove

\section*{Acknowledgments}
This work was supported by a public grant as part of the Investissements d'avenir project, reference ANR-11-LABX-0056-LMH, LabEx LMH and a fellowship via the International Exchange Program of \'Ecole Polytechnique.

\bibliographystyle{alphaurl}
\bibliography{ich_master,alles_ea_master,rest}

\end{document}